\documentclass[letterpaper,10pt,journal,twoside]{IEEEtran}
\usepackage[T1]{fontenc}
\usepackage{amsthm} % for definitions, theorems, etc.
\usepackage{amsmath,bm}
\usepackage{amsfonts} 
\usepackage{color}
\usepackage{makecell}
\usepackage{graphicx} 
\graphicspath{{./}}
\usepackage{multirow}
\usepackage{booktabs}
\usepackage{mathtools}
\usepackage{cite}
\usepackage{url}
\usepackage{hyperref}
\hypersetup{
  hidelinks,
  pdftitle={DiffCVaR: Reinforcement Learning for Risk Adaptation via Differentiable CVaR Barrier Functions},
  pdfauthor={Xinyi Wang, Taekyung Kim, Bardh Hoxha, Georgios Fainekos, and Dimitra Panagou},
  pdfsubject={Preprint},
  pdfkeywords={Planning under Uncertainty, Integrated Planning and Learning, Collision Avoidance}
}

\newtheorem{lemma}{Lemma}
\newtheorem{definition}{Definition}

\newtheorem{proposition}{Proposition}
\newtheorem{problem}{Problem}

\newtheorem{assumption}{Assumption}

\usepackage{booktabs}
\usepackage{multirow}
\usepackage[table]{xcolor}
\usepackage{arydshln}
\usepackage{makecell}

\definecolor{good}{RGB}{0,120,0}
\definecolor{bad}{RGB}{0,0,0}

\definecolor{optcolor}{HTML}{4E9F9A}
\definecolor{rlcolor}{HTML}{C9A227}
\definecolor{hybridcolor}{HTML}{D97959}
\newenvironment{revisionone}{\color{black}}{}
\newenvironment{revisiontwo}{\color{black}}{}

\newcommand{\meanstd}[2]{#1 {\tiny$\pm$ #2}}

\definecolor{bestred}{RGB}{200,40,40}
\definecolor{stripegray}{RGB}{242,242,242}
\newcommand{\best}[1]{\textcolor{good}{\textbf{#1}}}
\newcommand{\worst}[1]{\textcolor{bad}{\text{#1}}}

\newcommand{\shadefour}{\cellcolor{stripegray}}
\newcommand{\noshade}{}

\usepackage{acro}

\DeclareAcronym{cvar}{
  short = CVaR,
  long  = Conditional Value-at-Risk
}

\DeclareAcronym{cbf}{
  short = CBF,
  long  = Control Barrier Function
}

\DeclareAcronym{cvarbf}{
  short = CVaR-BF ,
  long  = CVaR barrier function
}
\DeclareAcronym{cvarbfqp}{
  short = CVaR-BF-QP ,
  long  = CVaR-BF quadratic program
}
\DeclareAcronym{qp}{
  short = QP ,
  long  = quadratic program
}
\DeclareAcronym{mpc}{
  short = MPC ,
  long  = model predictive control
}
\DeclareAcronym{gmm}{
  short = GMM ,
  long  = Gaussian mixture model
}
\DeclareAcronym{ood}{
  short = OOD ,
  long  = out-of-distribution
}
\DeclareAcronym{sr}{
  short = SR ,
  long  = success rate
}
\DeclareAcronym{cr}{
  short = CR ,
  long  = collision rate
}
\DeclareAcronym{rl}{
  short = RL,
  long  = reinforcement learning
}

\def\BibTeX{{\rm B\kern-.05em{\sc i\kern-.025em b}\kern-.08em
    T\kern-.1667em\lower.7ex\hbox{E}\kern-.125emX}}

\title{DiffCVaR: Reinforcement Learning for Risk Adaptation via Differentiable CVaR Barrier Functions}

\author{Xinyi Wang$^{1}$, Taekyung Kim$^{1}$, Bardh Hoxha$^{2}$, Georgios Fainekos$^{2}$ and Dimitra Panagou$^{1,3}$
\thanks{$^{1}$Department of Robotics, $^{3}$Department of Aerospace Engineering, University of Michigan, Ann Arbor, MI, 48109, USA {\tt\footnotesize \{taekyung, dpanagou\}@umich.edu}. Xinyi Wang was a postdoctoral researcher at the University of Michigan during this work and is now with the Department of Mechanical Engineering, Oakland University, Rochester, MI 48309 USA (e-mail: xinyiwang@oakland.edu).} 
\thanks{$^{2}$Toyota Motor North America, Research \& Development, Ann Arbor, MI, 48105, USA {\tt\footnotesize \{bardh.hoxha,georgios.fainekos\}@toyota.com}.}
}
\begin{document}

\maketitle

\begin{abstract}
Planning through crowded environments under uncertain obstacle motions remains difficult, as stochastic interactions often induce overly conservative behavior or reduced efficiency. To address this challenge, we propose an end-to-end risk adaptation framework for crowd navigation under obstacle-motion uncertainty modeled by a Gaussian mixture model. The framework combines reinforcement learning with a differentiable quadratic-program safety layer based on a Conditional Value-at-Risk barrier function, jointly learning nominal control input, risk level, and safety margin and enforcing explicit probabilistic safety constraints. 
This design enables context-aware adaptation, promoting efficient behavior while invoking caution only when necessary.
We conduct extensive evaluations in dynamic, uncertain, and crowded environments across varying obstacle densities, and also assess generalization under two out-of-distribution cases. 
Comparisons across optimization-based, RL-based, and integrated RL and optimization methods are provided, and the proposed method is shown to deliver the strongest overall performance in safety and generalization under uncertainty. \href{https://lawliet9666.github.io/rlcvarbf/}{[Paper Page]}\,
% \href{xxx}{[Video]}\,
\href{https://github.com/Lawliet9666/DiffCVaR.git}{[Code]}
\end{abstract}

\begin{IEEEkeywords}
Planning under Uncertainty, Integrated Planning and Learning, Collision Avoidance.
\end{IEEEkeywords}

\section{Introduction}
\IEEEPARstart{S}{afe} and efficient robot navigation in crowded environments under uncertain obstacle motion remains challenging. A key difficulty is achieving high efficiency while maintaining safety without excessive conservatism.
Risk-neutral objectives optimize expected cost but may overlook dangerous tail events~\cite{black2024risk}, whereas risk-averse robust optimization methods account for worst-case scenarios over distributions~\cite{xu2024distributionally} potentially at the price of excessive conservatism.

Recent work has adopted~\ac{cvar}, a tail-sensitive risk measure that provides a tunable balance between safety and efficiency~\cite{safaoui2024distributionally,ryu2024integrating, kim2025learning}.  Motivated by the deterministic guarantees of~\acp{cbf}, several works combine~\acp{cbf} with~\ac{cvar} to enforce probabilistic safety under uncertainty~\cite{kishida2024risk, chang2025risk, ahmadi2021risk}. 
However,~\ac{cvarbf} methods typically use a fixed risk level throughout the trajectory, limiting adaptability in dense crowds: conservative settings can cause overly cautious behavior or even infeasibility, whereas aggressive settings improve progress at the expense of safety. 
Prior work partially addresses this limitation by adapting the risk level online within the~\ac{cvarbf} framework~\cite{wang2025safe}. 
While promising, this approach relies on heuristics, and its performance in more complex settings, such as constrained dynamical robot models and multiple obstacle behavioral modes, remains unclear.
These limitations motivate a learning-based method, since adaptive risk parameters are difficult to design analytically across diverse stochastic crowd interactions.

Recently, \ac{rl} has shown advantages for adaptive decision making, since context-dependent control strategies can be learned directly from interaction data~\cite{choi2021risk, zhu2025confidence}. In addition, \ac{rl} has demonstrated strong performance by learning diverse and socially compliant behaviors in highly dynamic environments~\cite{liu2022intention, yao2024sonic, han2025dr}.
However, in safety-critical settings,~\ac{rl} policies typically provide only empirical safety and can degrade under distribution shift, such as denser crowds or previously unseen behaviors. 
\begin{figure}
  \centering
  \includegraphics[width=0.82\linewidth]{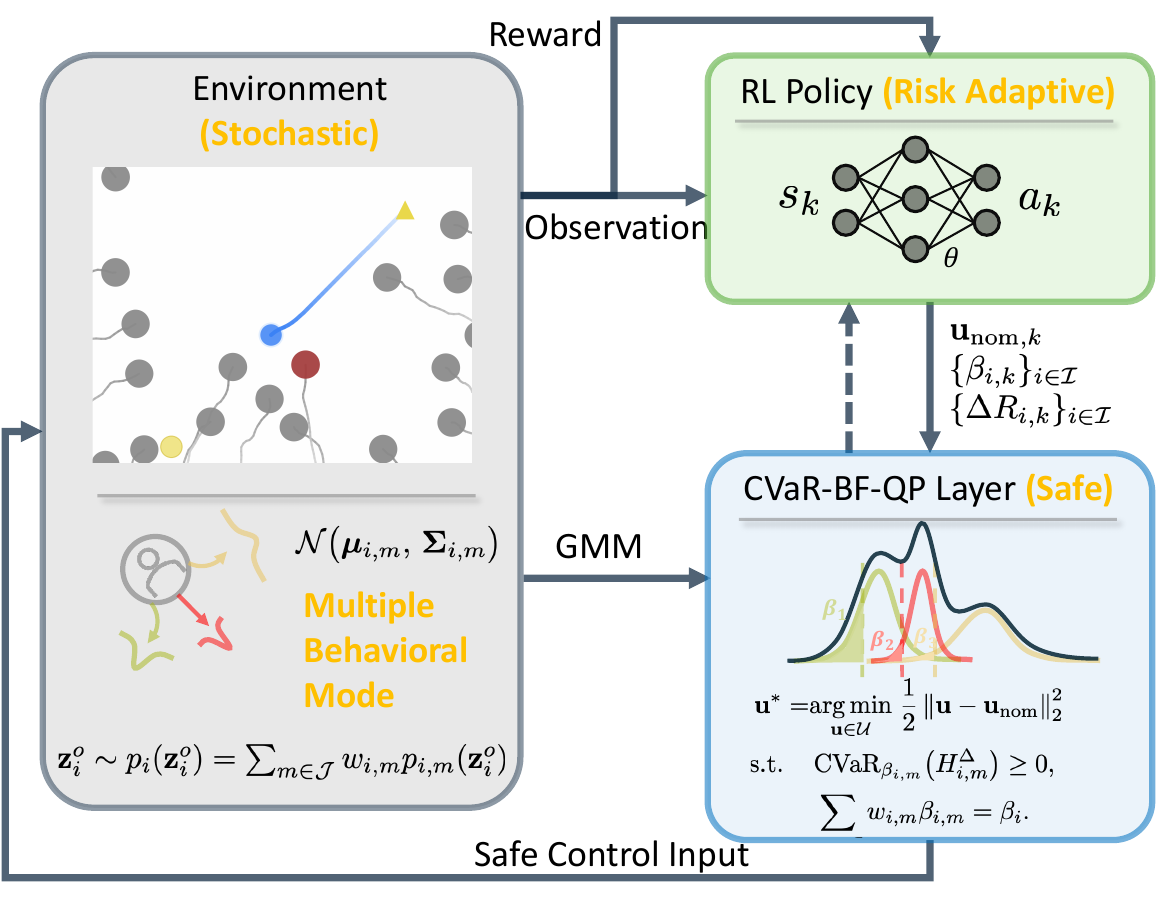}
\caption{Overview of DiffCVaR. The policy learns nominal control, risk level, and safety margin, while the differentiable GMM-based CVaR-BF-QP layer produces the applied safe control.}
\label{fig:overview}
\vspace{-10pt}
\end{figure}
To address this limitation, several mechanisms have been proposed to incorporate formal safety guarantees into~\ac{rl}~\cite{brunke2022safe}. 
A common approach is to use external safety modules, such as CBF-based safety filters that modify nominal policy outputs to satisfy safety constraints~\cite{gao2023online}, or model-predictive shielding methods that verify policy actions using predictive models and replace unsafe actions when necessary~\cite{tkbackup2026, banerjee2024dynamic}.
These methods can improve safety at inference time, but because the safety module is not jointly optimized with the pre-trained RL policy, the resulting closed-loop behavior can remain suboptimal. 
Instead, we adopt an integrated differentiable optimization layer that enables gradients to pass through this safety layer during end-to-end learning~\cite{amos2017optnet}. 
Closely related work such as~\cite{xiao2023barriernet, emam2021safe, wang2025multi} adopts a differentiable CBF layer, but assumes deterministic obstacle motion.

To address these gaps, we propose an end-to-end framework for safe navigation that integrates learning with a differentiable safety layer under uncertain obstacle motion. The main contributions are as follows:
\begin{itemize}
  \item An end-to-end risk-adaptive framework in which an~\ac{rl} policy jointly learns nominal control input, risk level, and safety margin for context-aware adaptation.
  \item A differentiable~\ac{cvarbfqp} safety layer under~\ac{gmm} obstacle-motion uncertainty, yielding tractable~\ac{qp} constraints with explicit probabilistic safety guarantees. To the best of our knowledge, this is the first tractable QP reformulation of a CVaR-based control barrier function under Gaussian-mixture uncertainty.
  \item A broad comparative study of optimization-based, \ac{rl}-based, and integrated~\ac{rl} and optimization methods in challenging dynamic environments, obstacle densities, and out-of-distribution cases, showing that the proposed method achieves the strongest overall performance in safety, efficiency, robustness, and generalization.
\end{itemize}

\section{Related Work}
\subsection{Robust Optimization}
Robust optimization provides a standard foundation for handling uncertainty and enforcing safety in dynamic environments. 
For example,~\cite{safaoui2024distributionally} developed a sample-based, distributionally robust~\ac{cvar} safety filter that constructs safe half-spaces from predicted obstacle trajectories for uncertain obstacle avoidance. \cite{ryu2024integrating} further incorporated predictive motion uncertainty into a distributionally robust, chance-constrained~\ac{mpc} framework for safe robot navigation in crowds.~\ac{cvar} has also been incorporated into~\ac{cbf} frameworks. For example,~\cite{kishida2024risk} integrated Kalman filtering and worst-case~\ac{cvar} to derive risk-aware safety constraints for discrete-time stochastic systems. More recently,~\cite{chang2025risk} combined cooperative sensing with a~\ac{cvarbfqp} safety filter for dynamic obstacle avoidance. 
These methods provide strong risk-aware safety guarantees, but they typically rely on fixed risk specifications and can become conservative in dense, highly dynamic interactions.

To mitigate this conservatism,~\cite{wang2025safe} introduced a risk-adaptive~\ac{cvarbf} that adjusts risk tolerance online and handles general uncertainty via sampling-based~\ac{cvar}. 
However, this formulation has two limitations. First, sampling-based \ac{cvar} calculation can be computationally demanding. Closed-form \ac{cvar} formulations address this issue by improving computational efficiency~\cite{safaoui2024distributionally, kishida2024risk}, but they often rely on Gaussian motion assumptions that do not capture multiple behavioral modes. 
Second, heuristic risk adaptation can be difficult to tune in scenarios involving more complex interactions and multiple obstacle behavioral modes. 
Learning-based methods address this issue by adapting risk sensitivity from interaction data~\cite{choi2021risk}, but they typically provide only empirical safety rather than explicit formal guarantees.
These limitations motivate an adaptive~\ac{cvarbf} method that is both computationally tractable and capable of providing real-time probabilistic
safety under~\ac{gmm} uncertainty.

\subsection{Safe Reinforcement Learning with Optimization}
\ac{rl} has become a prominent approach for learning risk-adaptive and socially compliant navigation behaviors in dynamic crowds~\cite{zhu2025confidence,liu2022intention,yao2024sonic}. Prior work has explored attention-based graph representations for modeling crowd interactions~\cite{liu2022intention,yao2024sonic}, adaptive conformal inference with constrained~\ac{rl} for socially aware navigation~\cite{yao2024sonic}, and confidence-aware robust dynamical-distance constrained~\ac{rl} for uncertain interactions~\cite{zhu2025confidence}. Despite strong empirical performance, these methods generally do not provide explicit probabilistic safety guarantees.

To improve safety, some approaches augment learned policies with runtime safety filters. For example, policy actions can be filtered through a~\ac{cbf}-\ac{qp} controller~\cite{gao2023online} or safeguarded by~\ac{mpc}-based shielding that verifies and replaces unsafe actions when necessary~\cite{tkbackup2026,banerjee2024dynamic}. Although these mechanisms can improve deployment-time safety, the safety layer is typically not jointly optimized with the pre-trained RL policy, which can lead to conservative or suboptimal closed-loop behavior.
A more integrated direction is to embed constrained optimization directly into the learning pipeline through differentiable layers~\cite{amos2017optnet}. 
Prior work leverages differentiable \ac{cbf}-\ac{qp} layers to facilitate gradient propagation during training~\cite{xiao2023barriernet,emam2021safe,wang2025multi}. Nevertheless, these methods rely on deterministic assumptions and do not explicitly incorporate uncertainty in an end-to-end framework.

\section{Preliminaries and Problem Formulation}
\subsection{Continuous-Time Control Barrier Functions}
Consider the continuous-time control-affine system:
\begin{equation}
\begin{aligned}
\label{eq:sys}
  \dot{\mathbf{x}}(t) &= f\big(\mathbf{x}(t)\big) + g\big(\mathbf{x}(t)\big)\mathbf{u}(t),
\end{aligned}
\end{equation}
where $\mathbf{u} \in \mathcal{U} \subset \mathbb{R}^{n_u}$ is the admissible control input and $\mathbf{x} \in \mathcal{X} \subseteq \mathbb{R}^{n_r}$ denotes the robot state. 
\begin{revisionone}
  Let $\mathbf{p} \in \mathbb{R}^{n_p}$ denote the robot's position.
\end{revisionone}
For notational simplicity, when the time dependence is clear from context, we write $\mathbf{x}$ and $\mathbf{u}$ in place of $\mathbf{x}(t)$ and $\mathbf{u}(t)$.
The functions $f : \mathcal{X} \to \mathbb{R}^{n_r}$ and $g:\mathcal{X}\to\mathbb{R}^{n_r\times n_u}$ are both assumed to be locally Lipschitz continuous.
Let $h: \mathbb{R}^{n_r} \to \mathbb{R}$ be a continuously differentiable function. The safe set is defined as the superlevel set
\begin{equation}
\label{eq:safeset}
\mathcal{X}_\textup{safe} = \{ \mathbf{x} \in \mathbb{R}^{n_r} \mid h(\mathbf{x}) \ge 0 \}.  
\end{equation}
A set $\mathcal{X}_\textup{safe} \subset \mathbb{R}^{n_r}$ is forward invariant for \eqref{eq:sys} under a control input $\mathbf{u}(t) \in \mathcal{U}$ if its solutions starting at any $\mathbf{x}(0) \in \mathcal{X}_\textup{safe}$ satisfy $\mathbf{x}(t) \in \mathcal{X}_\textup{safe}, \forall t \ge 0$.

\begin{definition}[Control Barrier Function~\cite{ames2019cbf}]
\label{def:cbf}
Let $\dot{h} = L_f h(\mathbf{x})+L_g h(\mathbf{x})\mathbf{u}$, 
where $L_f h(\mathbf{x})=\nabla h(\mathbf{x})^\top f(\mathbf{x})$ and $L_g h(\mathbf{x})=\nabla h(\mathbf{x})^\top g(\mathbf{x})$ are the Lie derivatives of the function $h$ with respect to $f$ and $g$.
Given the set $\mathcal{X}_\textup{safe}$ in \eqref{eq:safeset}, a continuously differentiable function $h:\mathbb{R}^{n_r}\to\mathbb{R}$ is a CBF for System~\eqref{eq:sys} if there exists an extended class-$\mathcal{K}_{\infty}$ function $\alpha$ such that
\begin{equation}
\sup_{\mathbf{u}\in\mathcal{U}}\left[L_f h(\mathbf{x})+L_g h(\mathbf{x})\mathbf{u}\right] \ge -\alpha\big(h(\mathbf{x})\big),
\quad \forall \mathbf{x}\in\mathcal{X}_\textup{safe}.
\end{equation}

\end{definition}

Following the~\ac{cbf} guarantee in \cite{ames2019cbf}, we define the admissible safe-control set $\mathcal{K}_{\textup{CBF}}(\mathbf{x}) := \{\mathbf{u}\in\mathcal{U} \mid L_f h(\mathbf{x}) + L_g h(\mathbf{x})\mathbf{u} + \alpha(h(\mathbf{x})) \ge 0\}$. If $h$ is a CBF and a locally Lipschitz controller $\kappa(\mathbf{x})$ satisfies $\kappa(\mathbf{x}) \in \mathcal{K}_{\textup{CBF}}(\mathbf{x})$ for all $\mathbf{x}\in\mathcal{X}_\textup{safe}$, then the set $\mathcal{X}_\textup{safe}$ is forward invariant.

\subsection{Probabilistic Constraints}
In dynamic-obstacle scenarios, due to the inherent uncertainty of obstacle motion, the safe set in \eqref{eq:safeset} cannot be computed deterministically. Let $\mathcal{I}\coloneqq\{1,\dots,N_{\textup{obs}}\}$ denote the set of obstacle indices, where $N_{\textup{obs}}$ is the number of obstacles observed. 
\begin{revisionone}
  For each $i\in\mathcal{I}$, let $\mathbf{x}^{o}_i\in\mathbb{R}^{n_{o}}$ denote the state and $\mathbf{p}_i^o \in \mathbb{R}^{n_p}$ denote the position.
\end{revisionone}
 For each obstacle, 
we use a continuously differentiable barrier function $h_i:\mathbb{R}^{n_r}\times\mathbb{R}^{n_{o}}\to\mathbb{R}$ and define the safe set as
\begin{equation}
\label{eq:s}
\mathcal{C}_i=\{\mathbf{x}\in\mathbb{R}^{n_r} \mid h_i(\mathbf{x},\mathbf{x}^{o}_i)\ge 0\}.
\end{equation}
The multi-obstacle safe set is the intersection $\mathcal{C}=\bigcap_{i\in\mathcal{I}}\mathcal{C}_i$.
\begin{revisionone}
Throughout the theoretical analysis below, we adopt a standard distance-based barrier function 
\begin{equation}
\label{eq:h}
h_i = \lVert\mathbf{p}-\mathbf{p}_i^o\rVert_2^2 - (R_r+R_i^o)^2
\end{equation}
 with $R_i^o$ and $R_r$ denoting the radius of obstacle $i$ and the robot, respectively.
\end{revisionone}
Following previous work~\cite{salzmann2020trajectron}, we assume $\mathbf{x}^{o}_i$ can be observed at the current time, while future obstacle motions are predicted under uncertainty, i.e.,
\begin{equation}
\dot{\mathbf{x}}^{o}_i=f^{o}_i\big(\mathbf{x}^{o}_i\big)+\mathbf{z}^{o}_i,
\end{equation}
where $f_i^o:\mathbb{R}^{n_o}\to\mathbb{R}^{n_o}$ 
denotes a prediction model of obstacle $i$, 
and for each fixed $t$, $\mathbf{z}^{o}_i$ is a random variable following an estimated distribution $\mathcal{D}_i$. 
In this work, we instantiate $\mathcal{D}_i$ as a~\ac{gmm} to capture multiple possible obstacle-motion modes. Let $\mathcal{J}\coloneqq\{1,\dots,M\}$ denote the set of \ac{gmm} mode indices:
\begin{equation}
\begin{aligned}
\label{eq:gmm}
\mathbf{z}^{o}_i\sim p_i(\mathbf{z}^{o}_i)
=\sum_{m\in\mathcal{J}} w_{i,m} p_{i,m}(\mathbf{z}^{o}_i),
\sum_{m\in\mathcal{J}}w_{i,m}=1.
\end{aligned}
\end{equation}
where each mode $p_{i,m}(\mathbf{z}^{o}_i)$ is a unimodal Gaussian distribution $\mathcal{N}(\bm{\mu}_{i,m},\bm{\Sigma}_{i,m})$ with mean vector $\bm{\mu}_{i,m}$ and covariance matrix $\bm{\Sigma}_{i,m}$ induced by the $m$-th Gaussian component. 

Taking the total time derivative of $h_i(\mathbf{x},\mathbf{x}^{o}_i)$ gives
\begin{equation}
\label{eq:h_dot}
\dot h_i(\mathbf{x},\mathbf{x}^{o}_i,\mathbf{z}^{o}_i)
=\nabla_{\mathbf{x}}h_i\big(\mathbf{x},\mathbf{x}^{o}_i\big)^\top\dot{\mathbf{x}}
+\nabla_{\mathbf{x}^{o}_i}h_i\big(\mathbf{x},\mathbf{x}^{o}_i\big)^\top\dot{\mathbf{x}}^{o}_i.
\end{equation}
Thus, at each time $t$, $\dot h_i$ is random and the~\ac{cbf} condition should be formulated as a probabilistic constraint.
Let
\begin{equation}
\label{eq:H}
H_i \coloneqq \dot h_i(\mathbf{x},\mathbf{x}^{o}_i, \mathbf{z}^{o}_i) +\alpha(h_i(\mathbf{x},\mathbf{x}^{o}_i)).
\end{equation} 
We next introduce the probabilistic~\ac{cbf} condition used to enforce safety under obstacle-motion uncertainty. 
\begin{definition}[Probabilistic~\ac{cbf}]
\label{def:pcbf}
Let $\beta_i \in (0,1)$ be the allowable violation probability for obstacle $i$. The collection of functions $\{h_i\}_{i\in\mathcal{I}}$ is a $\beta_i$-probabilistic \ac{cbf} if there exists an extended class-$\mathcal{K}$ function $\alpha(\cdot)$ such that for each $\mathbf{x} \in \mathcal{C}$ there exists $\mathbf{u} \in \mathbb{R}^{n_u}$ satisfying
\begin{equation}
\label{eq:pcbf}
\mathbb{P}\left(H_i \ge 0 \right) \ge 1-\beta_i,\quad \forall i\in\mathcal{I},
\end{equation}
where $\mathbf{z}^{o}_i \sim \mathcal{D}_i$ and $H_i$ is defined in \eqref{eq:H}.
\end{definition}
For notational simplicity, the following definitions are stated for a generic random variable $H$. In the multi-obstacle setting, they are applied to each $H_i$. Following~\cite{sarykalin2008value}, the associated Value-at-Risk (VaR) of $H$ is
\begin{equation}
\operatorname{VaR}_{\beta}(H)=\sup_{\zeta\in\mathbb{R}}\left\{\zeta\mid \mathbb{P}\bigl(H \ge \zeta\bigr)\ge 1-\beta\right\}.
\end{equation}
The~\ac{cvar} is then defined as follows:
\begin{definition}[Conditional Value-at-Risk~\cite{sarykalin2008value}]
\label{def:cvar}
The expected tail value of $H$ below the threshold 
$\operatorname{VaR}_{\beta}(H)$ is
\begin{equation}
\label{eq:cvar_ct}
\operatorname{CVaR}_{\beta}(H) \coloneqq \mathbb{E}\!\left[H \mid H \le \operatorname{VaR}_{\beta}(H)\right].
\end{equation}
\end{definition}
% An equivalent optimization form
% is given in \cite{sarykalin2008value}:
% \begin{equation}
% \label{eq:cvar_opt_ct}
% \operatorname{CVaR}_{\beta}(H)=-\inf_{\zeta\in\mathbb{R}}\mathbb{E}\!\left[\zeta+\frac{(-H -\zeta)_+}{\beta}\right],
% \end{equation}
% where $(\cdot)_+=\max\{\cdot,0\}$. 
Note that $\beta\to 1$ corresponds to the risk-neutral case, i.e., $\operatorname{CVaR}_{\beta}(H)\to\mathbb{E}[H]$; whereas $\beta\to 0$ corresponds to the risk-averse case, i.e., $\operatorname{CVaR}_{\beta}(H)\to \operatorname{VaR}_{\beta}(H)$ \cite{akella2024risk}.
Compared with VaR,~\ac{cvar} satisfies the axioms of a coherent risk measure, which are essential for rational risk assessment \cite{majumdar2019should}. 
It is also convex~\cite{sarykalin2008value}, which is important in our proposed method, as it enables a tractable convex optimization formulation.
The \ac{cvar} condition implies the chance constraint~\cite{wang2025safe}:
\begin{equation}
\label{eq:psafe}
\operatorname{CVaR}_{\beta}(H)\geq0
\ \Rightarrow\
\mathbb{P}(H\geq0)\geq1-\beta.
\end{equation}

\subsection{Problem Formulation}
We now formalize the finite-horizon navigation objective: to synthesize a control policy that maintains probabilistic safety under stochastic obstacle motion with multiple behavioral modes while efficiently reaching the navigation goal.
\begin{problem}
\label{prob:finite_horizon_navigation}
Given the initial state $\mathbf{x}(0)$, system dynamics \eqref{eq:sys}, obstacle uncertainty distributions $\{\mathcal{D}_i\}_{i\in\mathcal{I}}$, a time horizon $T>0$, and a target risk $\epsilon\in(0,1)$, find a task-efficient control policy satisfying
\begin{equation}
\label{eq:prob_safety}
\mathbb{P}\Big( \mathbf{x}(t) \in \mathcal{C},\ \forall t \in [0,T] \Big) \ge 1-\epsilon,
\end{equation}
where $\mathbf{x}(t)\in\mathcal{C}$ means $h_i(\mathbf{x}(t),\mathbf{x}_i^o(t))\ge0$ for all $i\in\mathcal{I}$.
\end{problem}
Directly enforcing the condition in~\eqref{eq:prob_safety} over the continuous horizon $[0,T]$ is intractable. 
Next, we introduce our method, which uses the probabilistic~\ac{cbf} condition in~\eqref{eq:pcbf} and reformulates these chance constraints into tractable mode-wise~\ac{cvar} constraints under the \ac{gmm} uncertainty model. Proposition~\ref{prop:finite_horizon_sampling} then shows that enforcing these constraints with a proper risk level guarantees~\eqref{eq:prob_safety}.

\section{Proposed Method}
This section presents the proposed risk-adaptive control method for Problem~\ref{prob:finite_horizon_navigation}. Using the \ac{gmm} introduced in~\eqref{eq:gmm}, the resulting chance constraints are reformulated as a~\ac{cvarbfqp} with probabilistic safety guarantees. This~\ac{cvarbfqp} is enforced through a differentiable optimization layer, enabling safe control and end-to-end policy learning.

\subsection{\ac{cvar} Constraint under~\ac{gmm}} 
Since $\dot h_i(\mathbf{x},\mathbf{x}^{o}_i,\mathbf{z}^{o}_i)$ in \eqref{eq:h_dot} is affine in $\mathbf{z}^{o}_i$, the $H_i$ in \eqref{eq:H} is also Gaussian under each GMM mode, i.e.,
\begin{equation}
\begin{aligned}
\label{eq:H_mode}
H_i\sim p_i(H_i) = \sum_{m\in\mathcal{J}} w_{i,m} p_{i,m}(H_i),\, \sum_{m\in\mathcal{J}}w_{i,m}=1,
\end{aligned}
\end{equation}
where each mode $p_{i,m}(H_i)$ is also a scalar Gaussian distribution $\mathcal{N}(\mu_{i,m}^H,(\sigma_{i,m}^H)^2)$ with mean and variance
\begin{equation}
\begin{split}
\label{eq:H_mode_mu_cov}
&\mu_{i,m}^H =
\nabla_{\mathbf{x}} h_i(\mathbf{x},\mathbf{x}^{o}_i)^\top
\big(f(\mathbf{x})+g(\mathbf{x})\mathbf{u}\big)
+ \alpha(h_i(\mathbf{x},\mathbf{x}^{o}_i))\\
&\quad
+ \nabla_{\mathbf{x}^{o}_i}h_i(\mathbf{x},\mathbf{x}^{o}_i)^\top
\big(f^{o}_i(\mathbf{x}^{o}_i)+\bm{\mu}_{i,m}\big),\\
&(\sigma_{i,m}^H)^2 =
\nabla_{\mathbf{x}^{o}_i}h_i(\mathbf{x},\mathbf{x}^{o}_i)^\top
\bm{\Sigma}_{i,m}
\nabla_{\mathbf{x}^{o}_i}h_i(\mathbf{x},\mathbf{x}^{o}_i).
\end{split}
\end{equation}
We use the linear extended
class-$\mathcal{K}$ function $\alpha(h)=\alpha_0h$, where $\alpha_0>0$.
In previous work, under a~\ac{gmm} distribution,~\ac{cvar} can be computed following~\cite{yang2024risk}, but it requires a numerical search for the VaR of the~\ac{gmm}, which increases the online computational cost.
\begin{revisionone}
We instead impose the following mode-wise \ac{cvar}, which is a tractable sufficient surrogate for the full \ac{gmm} \ac{cvar} constraint:
\end{revisionone}
\begin{equation}
\operatorname{CVaR}_{\beta_{i,m}}(H_{i,m}) \ge 0, \quad \forall m\in\mathcal{J}.
\label{eq:gmm_mode_cvar}
\end{equation}
For each $H_{i,m}$, \ac{cvar} admits a closed-form expression \cite{norton2021calculating}:
\begin{equation}
\label{eq:cvar_one_gaussian}
\operatorname{CVaR}_{\beta_{i,m}}(H_{i,m})=\mu_{i,m}^H-\frac{\phi(\Phi^{-1}(\beta_{i,m}))}{\beta_{i,m}}\sigma_{i,m}^H,
\end{equation}
where $\phi$ and $\Phi^{-1}$ are the standard normal PDF and inverse CDF, respectively.
Since $\mu_{i,m}^H$ is affine in $\mathbf{u}$ and $\sigma_{i,m}^H$ is independent of $\mathbf{u}$, each constraint \(\operatorname{CVaR}_{\beta_{i,m}}(H_{i,m})\ge 0\) is an affine inequality in $\mathbf{u}$. 
\begin{revisiontwo}
Compared with tail-probability bounds, such as Chebyshev–Cantelli and Vysochanskij–Petunin, Gaussian CVaR can be less conservative by exploiting each mode’s Gaussian distribution.
\end{revisiontwo}

\begin{revisionone}
\begin{lemma}[Mode-Wise Chance Constraint]
\label{lem:gaussian_cvar_chance}
For any Gaussian mode
\(H_{i,m}\sim\mathcal{N}(\mu_{i,m}^H,(\sigma_{i,m}^H)^2)\), if
\eqref{eq:gmm_mode_cvar} holds for some \(\beta_{i,m}\in(0,1)\), then
\begin{equation}
\label{eq:modewise_tightened_chance}
\mathbb{P}_{p_{i,m}}(H_{i,m}\geq0)
\geq
1-\widehat{\beta}(\beta_{i,m})
\geq
1-\beta_{i,m},
\end{equation}
where the bound is
$
\widehat{\beta}(b)
\coloneqq
\Phi\!\left(
-\frac{\phi(\Phi^{-1}(b))}{b}
\right),
\, b\in(0,1).
$
\end{lemma}

\begin{proof}
\begingroup
\small
\setlength{\jot}{1pt}
\setlength{\abovedisplayskip}{3pt}
\setlength{\belowdisplayskip}{3pt}
Let \(Z\sim\mathcal{N}(0,1)\) and \(b\coloneqq\beta_{i,m}\). Under the
definition in \eqref{eq:cvar_ct}, and since \(\Phi\) is monotonically
increasing,
\[
\begin{aligned}
\operatorname{CVaR}_{b}(Z)
&=
-\frac{\phi(\Phi^{-1}(b))}{b}
\leq \operatorname{VaR}_{b}(Z)
=
\Phi^{-1}(b),
\end{aligned}
\]
and hence,
\[
\begin{aligned}
\widehat{\beta}(b)
&=
\Phi\!\left(\operatorname{CVaR}_{b}(Z)\right)
\leq
\Phi\!\left(\operatorname{VaR}_{b}(Z)\right)
=
b.
\end{aligned}
\]
For \(\sigma_{i,m}^H>0\), the standardized mode
\((H_{i,m}-\mu_{i,m}^H)/\sigma_{i,m}^H\) has the same distribution as \(Z\).
Equations~\eqref{eq:gmm_mode_cvar} and \eqref{eq:cvar_one_gaussian} give
$\frac{\mu_{i,m}^H}{\sigma_{i,m}^H}\geq
\frac{\phi(\Phi^{-1}(\beta_{i,m}))}{\beta_{i,m}}$. Thus,
\[
\begin{aligned}
\mathbb{P}_{p_{i,m}}(H_{i,m}<0)
&=
\mathbb{P}_{p_{i,m}}\!\left(
Z<
-\frac{\mu_{i,m}^H}{\sigma_{i,m}^H}
\right) =
\Phi\!\left(-\frac{\mu_{i,m}^H}{\sigma_{i,m}^H}\right)
\\
&\leq
\Phi\!\left(
-\frac{\phi(\Phi^{-1}(\beta_{i,m}))}{\beta_{i,m}}
\right)
=
\widehat{\beta}(\beta_{i,m})
\leq
\beta_{i,m}.
\end{aligned}
\]
If \(\sigma_{i,m}^H=0\), \eqref{eq:gmm_mode_cvar} implies \(\mu_{i,m}^H\geq0\), so
\(\mathbb{P}_{p_{i,m}}(H_{i,m}<0)=0\).
\endgroup
\end{proof}
\end{revisionone}

\begin{revisionone}
\begin{proposition}[GMM Mixture Chance Constraint]
\label{prop:modewise_cvar}
Suppose that \eqref{eq:gmm_mode_cvar} holds for obstacle \(i\) and the
mode-wise risk allocation satisfies
\begin{equation}
\sum_{m\in\mathcal{J}}w_{i,m}\beta_{i,m}=\beta_i.
\label{eq:gmm_risk_budget}
\end{equation}
Then,
\begin{equation}
\label{eq:gmm_tightened_chance}
\begin{aligned}
\mathbb{P}(H_i\geq0)
\geq
1-\sum_{m\in\mathcal{J}}
w_{i,m}\widehat{\beta}(\beta_{i,m})\geq
1-\beta_i.
\end{aligned}
\end{equation}
Thus, \eqref{eq:pcbf} holds for obstacle \(i\).
\end{proposition}

\begin{proof}
By Lemma~\ref{lem:gaussian_cvar_chance}, the law of total probability yields
\begin{align*}
\mathbb{P}(H_i\geq0)
&=
\sum_{m\in\mathcal{J}}
w_{i,m}\mathbb{P}_{p_{i,m}}(H_{i,m}\geq0)\\
&\geq
1-\sum_{m\in\mathcal{J}}w_{i,m}\widehat{\beta}(\beta_{i,m})
\\
&\geq
1-\sum_{m\in\mathcal{J}}w_{i,m}\beta_{i,m}
=1-\beta_i.
\qedhere
\end{align*}
\end{proof}
% \vspace{-10pt}

Following~\cite{ren2022chance}, our implementation uses the uniform mode
allocation \(\beta_{i,m}=\beta_i\) for all \(m\in\mathcal{J}\). 
\begin{revisiontwo}
At step \(k\), all modes of obstacle \(i\) share \(\beta_{i,k}\), while different active obstacles may receive different risk levels.
\end{revisiontwo}
Since
\(\sum_{m\in\mathcal{J}}w_{i,m}=1\), it satisfies
\eqref{eq:gmm_risk_budget}, and the weighted tightened bound
simplifies to
\(\sum_{m\in\mathcal{J}}w_{i,m}\widehat{\beta}(\beta_i)
=\widehat{\beta}(\beta_i)\).
\end{revisionone}

\subsection{Probabilistic Safety Guarantee}
Building on Proposition~\ref{prop:modewise_cvar}, we formulate the following tractable~\ac{cvarbfqp}.
\begin{revisionone}
Given a nominal input $\mathbf{u}_{\textup{nom}}$, risk
levels $\{\beta_i\}_{i\in\mathcal{I}}$, and risk allocations
satisfying \eqref{eq:gmm_risk_budget}, this yields the following~\ac{cvarbfqp},
\end{revisionone}
\begin{equation}
\begin{split}
\raisetag{6.0ex}
\label{eq:cvarbfqp}
\mathbf{u}^*=  \arg \min_{\mathbf{u}\in\mathcal{U}} \frac{1}{2}\left\|\mathbf{u}-\mathbf{u}_{\textup{nom}}\right\|_2^2 \quad
\textup{s.t.} \eqref{eq:gmm_mode_cvar} \, \text{and} \, \eqref{eq:gmm_risk_budget}.
\end{split}
\end{equation}

We now prove that solving~\eqref{eq:cvarbfqp} yields the finite-horizon probabilistic safety guarantee in~\eqref{eq:prob_safety} under Assumption~\ref{assump:inter_sample}. In implementation, since the controller is actually applied at discrete sampling instants, we partition the continuous horizon $[0,T]$ into sampling intervals and prove safety recursively over these intervals.
Let $t_k=k\Delta t$ for $k=0,\dots,K-1$ and $t_K=T$, where $\Delta t>0$ and $K=\lceil T/\Delta t\rceil$. 
For each $k$, define the safety event up to time $t_k$ as
\begin{equation}
\label{eq:finite_horizon_event}
\Gamma_k\coloneqq
\bigcap_{i\in\mathcal{I}}
\left\{
h_i(\mathbf{x}(t),\mathbf{x}_i^o(t))\ge0,\ 
\forall t\in[0,t_k]
\right\}.
\end{equation}

\begin{assumption}
\label{assump:inter_sample}
\begin{revisionone}
Under the uniform mode allocation used in our implementation, the probabilistic~\ac{cbf} bound in~\eqref{eq:gmm_tightened_chance}, enforced at
sampling time \(t_k\), is assumed to remain satisfied over the interval
\([t_k,t_{k+1})\), conditioned on \(\Gamma_k\), i.e.,
$
\mathbb{P}\Bigl(
H_i(t)\geq0,\ \forall t\in[t_k,t_{k+1})
\,\Big|\,\Gamma_k
\Bigr)
\geq1-\widehat{\beta}(\beta_{i,k}).
$
\end{revisionone}
\end{assumption}
Assumption~\ref{assump:inter_sample} can be verified under standard sampled-data CBF assumptions~\cite{hoxha2026bayesian}: (i)~zero-order-hold control and (ii)~Lipschitz continuity of $H_i$ along trajectories, $|H_i(t)-H_i(t_k)|\le L_{i,k}|t-t_k|$. These imply that inter-sample safety holds when the sampling interval $\Delta t$ is sufficiently small that $\mathbb{P}(H_i(t_k)\ge L_{i,k}\Delta t\mid\Gamma_k)\ge \textcolor{black}{1-\widehat{\beta}(\beta_{i,k})}$. 
\begin{revisionone}
Adapting the discrete-time argument of~\cite[Proposition~1]{mestres2025pcbf},
we obtain the following continuous-time result.
\end{revisionone}

\begin{proposition}[Probabilistic safety over a finite horizon]
\label{prop:finite_horizon_sampling}
Suppose Assumption~\ref{assump:inter_sample} holds and
\(\mathbf u_k\) is computed by~\eqref{eq:cvarbfqp} at every \(t_k\).
\begin{revisionone}
If \(\mathbb{P}(\Gamma_0)=1\) and there exists a \(\beta\in(0,1)\) such that
$ \sum_{i\in\mathcal{I}}\beta_{i,k}
\leq\beta,
\, k\in\{0,\dots,K-1\},$
then
\end{revisionone}
\[
\begin{aligned}
\mathbb{P}(\Gamma_K)
&\geq
\textcolor{black}{\prod_{k=0}^{K-1}
\left(1-\sum_{i\in\mathcal I}
\widehat{\beta}(\beta_{i,k})\right)}\geq(1-\beta)^K.
\end{aligned}
\]
In particular, for any target horizon risk \(\epsilon\in(0,1)\), if
\(\beta\le1-(1-\epsilon)^{1/K}\), then
\(\mathbb{P}(\Gamma_K)\ge1-\epsilon\).
\end{proposition}

\begin{proof}
According to Proposition~\ref{prop:modewise_cvar}, the solution
of~\eqref{eq:cvarbfqp} enforces
\textcolor{black}{the bound in~\eqref{eq:gmm_tightened_chance}}
for every $k\in\{0,\dots,K-1\}$. For each obstacle $i\in\mathcal{I}$, define
\(E_{k,i}\coloneqq\{H_i(t)\ge0,\ \forall t\in[t_k,t_{k+1})\}\),
and let $E_k\coloneqq\bigcap_{i\in\mathcal{I}}E_{k,i}$. By Assumption~\ref{assump:inter_sample},
\(\mathbb{P}(E_{k,i}\mid\Gamma_k)
\geq\textcolor{black}{1-\widehat{\beta}(\beta_{i,k})}\) for every
\(i\in\mathcal{I}\).
Thus, by De Morgan's law and the union bound,
\[
\begin{aligned}
\mathbb P(E_k\mid\Gamma_k)
&=1-\mathbb P(E_k^c\mid\Gamma_k)=1-\mathbb P\!\left(\bigcup_{i\in\mathcal{I}}E_{k,i}^c\,\middle|\,\Gamma_k\right)\\
&\ge 1-\sum_{i\in\mathcal{I}}\mathbb P(E_{k,i}^c\mid\Gamma_k)\geq \textcolor{black}{1-\sum_{i\in\mathcal{I}}
\widehat{\beta}(\beta_{i,k})}
\\
&\geq \textcolor{black}{1-\sum_{i\in\mathcal{I}}
\beta_{i,k}}
\geq \textcolor{black}{1-\beta}.
\end{aligned}
\]
Define the interval safety events
\(F_{k,i}\coloneqq
\{h_i(\mathbf{x}(t),\mathbf{x}_i^o(t))\ge0,
\forall t\in[t_k,t_{k+1})\}\) and
\(F_k\coloneqq\bigcap_{i\in\mathcal{I}}F_{k,i}\).
By the~\ac{cbf} theory~\cite{ames2019cbf}, we have $\Gamma_k\cap E_{k,i}\subseteq F_{k,i}$. Taking the intersection over all $i\in\mathcal{I}$ gives
\(\Gamma_k\cap E_k\subseteq F_k\).
Since $\Gamma_{k+1}=\Gamma_k\cap F_k$, we have
\[
\begin{aligned}
\mathbb P(\Gamma_{k+1}\mid\Gamma_k)
&=\mathbb P(F_k\mid\Gamma_k)\geq \mathbb P(E_k\mid\Gamma_k)\\
&\geq\textcolor{black}{1-\sum_{i\in\mathcal I}
\widehat{\beta}(\beta_{i,k})}\geq1-\beta.
\end{aligned}
\]
Since $\Gamma_{k+1}$ augments $\Gamma_k$ with safety over the entire interval $[t_k,t_{k+1})$, the events $\{\Gamma_k\}_{k=0}^{K}$ are nested continuous-time safety events. Thus, using $\mathbb P(\Gamma_0)=1$, the chain rule gives
\[
\begin{aligned}
\mathbb P(\Gamma_K)
&\geq\textcolor{black}{\prod_{k=0}^{K-1}
\left(1-\sum_{i\in\mathcal I}
\widehat{\beta}(\beta_{i,k})\right)}\geq(1-\beta)^K.
\end{aligned}
\]
If $\beta\le 1-(1-\epsilon)^{1/K}$, then $(1-\beta)^K\ge 1-\epsilon$, and hence $\mathbb P(\Gamma_K)\ge 1-\epsilon$.
Since $t_K=T$, \eqref{eq:finite_horizon_event} gives the desired safety
event over $[0,T]$.
\end{proof}
  % \vspace{-10pt}

\subsection{Adaptive~\ac{cvarbfqp}}
\label{sec:adaptive_cvarbfqp}
The proposed optimization in~\eqref{eq:cvarbfqp} can be implemented as a differentiable safety layer, enabling \textbf{end-to-end policy learning} while maintaining \textbf{probabilistic safety guarantees}. 
\begin{revisionone}
The resulting closed-loop system is modeled as a Markov decision process 
and trained with an actor-critic algorithm. 
At sampling step $k$, to encode interactions with multiple dynamic obstacles, we represent the
observation $s_k\in\mathcal{S}$ as a spatio-temporal graph
$\mathcal{G}_k=(\mathcal{V}_k,\mathcal{E}_k)$, following the
attention-based method in~\cite{kim2026oacbf}. Unlike
\cite{kim2026oacbf}, our encoder
connects corresponding entities across consecutive observations using
temporal edges, providing short-term context for changes in relative geometry
and obstacle motion.

Unlike standard~\ac{rl}, the action $a_k\in\mathcal{A}$ in our framework is not the final control input applied to the system. Instead, it is sent into the~\ac{cvarbfqp} layer, which solves for the safe control input $\mathbf{u}_k^*$. 
The graph encoder produces a robot embedding and individual obstacle embeddings. 
The actor network parameterized by $\theta$ maps the robot embedding
through a nominal-control MLP head to the nominal control input
$\mathbf{u}_{\textup{nom},k}$ and applies shared obstacle MLP heads
to each obstacle embedding to produce the risk levels
$\{\beta_{i,k}\}_{i\in\mathcal{I}}$ and adaptive safety-margin offsets
$\{\Delta R_{i,k}\}_{i\in\mathcal{I}}$.
These outputs define
$a_k\coloneqq\bigl(\mathbf{u}_{\textup{nom},k},
\{\beta_{i,k}\}_{i\in\mathcal{I}},
\{\Delta R_{i,k}\}_{i\in\mathcal{I}}\bigr)$,
while the critic estimates the state value $V_\omega(s_k)$.
\end{revisionone}

\begin{revisionone}
We constrain the learned parameters to
\(\beta_{i,k}\in(0,1)\),
\(\sum_{i\in\mathcal{I}}\beta_{i,k}\leq\beta\), and
\(\Delta R_{i,k}\in[\Delta R_{\min},\Delta R_{\max}]\), where
\(\beta\in(0,1)\) is a user-defined total risk budget and
\(\Delta R_{\min},\Delta R_{\max}\) are the minimum and maximum admissible
safety-margin offsets, respectively.
\end{revisionone}
\begin{revisionone}
\begin{samepage}
At sampling time \(t_k\), the fixed policy output
\(\Delta R_{i,k}\geq0\) defines the following barrier function
\begin{equation}
\label{eq:adaptive_buffered_barrier}
h_{i,k}^{\Delta}(\mathbf{p},\mathbf{p}_i^o)
=
\|\mathbf{p}-\mathbf{p}_i^o\|_2^2
-
\bigl(R_r+R_i^o+\Delta R_{i,k}\bigr)^2.
\end{equation}
\end{samepage} 
Let $H_{i,m,k}^{\Delta}$ denote the mode-wise CBF expression constructed
from $h_{i,k}^{\Delta}$.
Define
\(c_{i,k}\coloneqq
2(R_r+R_i^o)\Delta R_{i,k}+(\Delta R_{i,k})^2\).
Since \(\Delta R_{i,k}\geq0\), we have \(c_{i,k}\geq0\) and
\(H_{i,m,k}^{\Delta}
=H_{i,m,k}-\alpha_0c_{i,k}\leq H_{i,m,k}\).
Therefore, by CVaR translation equivariance,
\begin{align}
\operatorname{CVaR}_{\beta_{i,m,k}}\bigl(H_{i,m,k}^{\Delta}\bigr)
&\geq0
\label{eq:tightened_margin_cvar}\\
\Longrightarrow\, \operatorname{CVaR}_{\beta_{i,m,k}}\bigl(H_{i,m,k}\bigr)
&\geq\alpha_0c_{i,k}\geq0.
\label{eq:physical_barrier_cvar}
\end{align}
Hence, \eqref{eq:tightened_margin_cvar} implies
\eqref{eq:gmm_mode_cvar} at \(t_k\) and, by
Proposition~\ref{prop:finite_horizon_sampling}, the finite-horizon safety guarantee.
Because the~\ac{cvarbfqp} layer is differentiable, gradients from the  policy loss $\mathcal L_\pi$ backpropagate through the safety layer to the actor parameters $\theta$. At sampling step $k$, the contribution through the QP solution $\mathbf u_k^*$ is
\begin{equation}
\begin{aligned}
\frac{\partial\mathcal L_\pi}{\partial\mathbf u_k^*}
\biggl[
&\frac{\partial\mathbf u_k^*}{\partial\mathbf u_{\textup{nom},k}}
 \frac{\partial\mathbf u_{\textup{nom},k}}{\partial\theta}+\sum_{i\in\mathcal I}
 \frac{\partial\mathbf u_k^*}{\partial\beta_{i,k}}
 \frac{\partial\beta_{i,k}}{\partial\theta}\\
&+\sum_{i\in\mathcal I}
 \frac{\partial\mathbf u_k^*}{\partial\Delta R_{i,k}}
 \frac{\partial\Delta R_{i,k}}{\partial\theta}
\biggr].
\end{aligned}
\end{equation}
This enables the policy to optimize task performance while enforcing the safety constraints whenever feasible. In practice, fast-moving, uncooperative obstacles may render constraints infeasible. Therefore, the deployed implementation uses slack variables to maintain \ac{qp} feasibility.

\end{revisionone}

\begin{revisionone}
\begin{table*}[t]
\centering
% \small
\footnotesize
\setlength{\tabcolsep}{5pt}
\renewcommand{\arraystretch}{1.15}
\caption{Comparisons of baseline methods for the unicycle robot model in environments with 20 obstacles.}
\begin{tabular}{llcccccc}
\hline
\textbf{Type} & \textbf{Policy}
& \textbf{SR} & \textbf{CR} & \textbf{TR} & \textbf{Avg. Ret.} & \textbf{Traj. Time} & \textbf{Min. Dist.} \\
\hline

\multirow{4}{*}{Opt.}
& ORCA~\cite{alonso2013optimal}
& \noshade 52.6
& \noshade 47.4
& \noshade 0.0
& \noshade\meanstd{1.92}{0.24}
& \noshade\best{\meanstd{5.76}{0.13}}
& \noshade\meanstd{0.79}{0.09} \\
& CBF-QP~\cite{ames2019cbf}
& \shadefour 57.0
& \shadefour 43.0
& \shadefour 0.0
& \shadefour\meanstd{2.25}{0.18}
& \shadefour\meanstd{6.60}{0.13}
& \shadefour\meanstd{0.75}{0.09} \\
& \ac{cvarbfqp}~\cite{ahmadi2021risk}
& \noshade 54.6
& \noshade 45.4
& \noshade 0.0
& \noshade\meanstd{2.11}{0.21}
& \noshade\meanstd{6.63}{0.14}
& \noshade\meanstd{0.80}{0.10} \\
& Adaptive-CVaR-BF~\cite{wang2025safe}
& \shadefour 52.0
& \shadefour 48.0
& \shadefour 0.0
& \shadefour\meanstd{2.01}{0.22}
& \shadefour\meanstd{6.52}{0.15}
& \shadefour\meanstd{0.82}{0.10} \\

\noalign{\vskip 2pt}
\cdashline{1-2}\cdashline{3-8}
\noalign{\vskip 2pt}

\multirow{3}{*}{RL}
& Vanilla~\ac{rl}
& \noshade 58.2
& \noshade 41.8
& \noshade 0.0
& \noshade\meanstd{2.07}{0.20}
& \noshade\meanstd{6.17}{0.10}
& \noshade\best{\meanstd{0.92}{0.12}} \\
& CrowdNav++ (const vel)~\cite{liu2022intention}
& \shadefour\worst{48.8}
& \shadefour\worst{51.2}
& \shadefour 0.0
& \shadefour\worst{\meanstd{0.59}{0.10}}
& \shadefour\meanstd{7.84}{0.23}
& \shadefour\worst{\meanstd{0.42}{0.05}} \\
& CrowdNav++ (inferred)~\cite{liu2022intention}
& \noshade 63.6
& \noshade 36.4
& \noshade 0.0
& \noshade\meanstd{0.73}{0.12}
& \noshade\meanstd{7.70}{0.24}
& \noshade\meanstd{0.52}{0.06} \\

\noalign{\vskip 2pt}
\cdashline{1-2}\cdashline{3-8}
\noalign{\vskip 2pt}

\multirow{5}{*}{RL+Opt.}
& Vanilla~\ac{rl} + SF
& \shadefour 65.0
& \shadefour 35.0
& \shadefour 0.0
& \shadefour\meanstd{1.96}{0.22}
& \shadefour\meanstd{6.78}{0.09}
& \shadefour\meanstd{0.90}{0.09} \\
& CrowdNav++ (const vel) + SF
& \noshade 63.8
& \noshade 36.0
& \noshade 0.2
& \noshade\meanstd{0.65}{0.12}
& \noshade\worst{\meanstd{8.84}{0.28}}
& \noshade\meanstd{0.53}{0.05} \\
& CrowdNav++ (inferred) + SF
& \shadefour 68.6
& \shadefour 31.4
& \shadefour 0.0
& \shadefour\meanstd{0.70}{0.14}
& \shadefour\meanstd{8.16}{0.29}
& \shadefour\meanstd{0.60}{0.06} \\
\cdashline{2-2}[0.3pt/1.5pt]\cdashline{3-8}[0.3pt/1.5pt]
& BarrierNet~\cite{xiao2023barriernet}
& \noshade 83.2
& \noshade 16.2
& \noshade\worst{0.6}
& \noshade\meanstd{2.35}{0.19}
& \noshade\meanstd{6.51}{0.18}
& \noshade\meanstd{0.75}{0.09} \\
& Proposed
& \shadefour\best{97.0}
& \shadefour\best{2.6}
& \shadefour 0.4
& \shadefour\best{\meanstd{2.38}{0.21}}
& \shadefour\meanstd{6.86}{0.12}
& \shadefour\meanstd{0.70}{0.08} \\
\hline
\end{tabular}
\label{tab:baseline}
\vspace{-15pt}
\end{table*}
\end{revisionone}

\section{Simulation}

\subsection{Simulation Settings}
\noindent\textbf{Environment.}
We evaluate our method in a crowd-navigation simulator on a $12\,\textup{m} \times 12\,\textup{m}$ workspace with time step $\Delta t = 0.1\,\textup{s}$. 
We train our policy using 20 dynamic obstacles, each modeled as a circle of radius $0.4\,\textup{m}$, with speeds sampled from $[0.5,1.5]\,\textup{m/s}$ and randomized initial positions and velocities. Test results are averaged over 10 independent runs (random seeds), each evaluated on 50 episodes with randomized initial conditions.

Obstacle motion follows the Social Force Model (SFM) with uncooperative behavior. 
Each obstacle uses the three-mode~\ac{gmm} in~\eqref{eq:gmm}: one forward mode and two lateral-deviation modes. For simplicity, all obstacles use the same~\ac{gmm}, with weights $[w_1,w_2,w_3]=[0.6,0.2,0.2]$ and standard deviations $[\sigma_1,\sigma_2,\sigma_3]=[0.1,0.2,0.2]$, although the framework supports obstacle-dependent parameters.
The forward mode follows the SFM direction, whereas the two lateral modes represent sudden left or right deviations. Across multiple obstacles, these modes induce combinatorial uncertainty that can accumulate over time and create tail-risk collisions absent under deterministic obstacle-motion assumptions in prior environments~\cite{liu2022intention, yao2024sonic}. 
To isolate the contribution of the~\ac{cvarbfqp} layer, we assume oracle access to obstacle uncertainty, i.e., the simulator provides~\ac{gmm} parameters that are propagated with constant-velocity prediction and then passed to the~\ac{cvar} computation module.
In real deployments, the same interface can be provided by a learned predictor, such as Trajectron++~\cite{salzmann2020trajectron}.

\noindent\textbf{Robot models.}
Consider a unicycle ($\dot x=v\cos\theta$, $\dot y=v\sin\theta$, $\dot\theta=\omega$, $\mathbf{u}=[v,\omega]^\top$), with look-ahead point $\mathbf{p}_{\ell}=[x+\ell\cos\theta,\,y+\ell\sin\theta]^\top$. Set $v_{\max}=1.5$\,m/s and $|\omega|\le\omega_{\max}=1.57$\,rad/s.
\begin{revisionone}
Collision avoidance with obstacle $i$ uses Eq.~\eqref{eq:adaptive_buffered_barrier}, with
$\mathbf{p}=\mathbf{p}_\ell$ and $R_r+\ell$ as radius for the unicycle~\cite{wilson2020robotarium}.
\end{revisionone} 

\noindent\textbf{Observation \& action space.}
\begin{revisionone}
For each observation frame $\tau$, we define $\bar{o}_\tau
= [
(\bar{o}_\tau^{\textup{rg}})^\top,
(\bar{o}_\tau^{1})^\top,\ldots,
(\bar{o}_\tau^{N_{\textup{obs}}})^\top]^\top$ with
$N_{\textup{obs}}=20$, where the robot-goal and obstacle observations are
$
\bar{o}_\tau^{\textup{rg}}
=[
(\mathbf{p}_g-\mathbf{p}_\tau)^\top,
\mathbf{v}_\tau^\top,\theta_\tau,R_r]^\top
\in\mathbb{R}^{6}$, 
$\bar{o}_\tau^{i}
=[
(\mathbf{p}_\tau-\mathbf{p}_{i,\tau}^{o})^\top,
(\mathbf{v}_{i,\tau}^{o})^\top,R_i^o,\begin{revisiontwo}\chi_{i,\tau}\end{revisiontwo}
]^\top
\in\mathbb{R}^{6}.
$
Here, $\begin{revisiontwo}\chi_{i,\tau}\end{revisiontwo}\in\{0,1\}$ is the activity mask for obstacle $i$.  
The observation at step $k$ is then defined as
$s_k:=[\bar{o}_{k-2}^\top,\bar{o}_{k-1}^\top,\bar{o}_k^\top]^\top.$
The graph encoder uses three attention layers with hidden dimension 64 and produces 16-dimensional node embeddings.
All observed obstacles are encoded by the policy and \(\mathcal I_k\subseteq\mathcal I\) denotes the \(N_{\mathrm{QP}}=3\) nearest active obstacles selected by the safety layer. 
\end{revisionone}
As defined in Section~\ref{sec:adaptive_cvarbfqp}, the policy action parameterizes the differentiable safety layer. The policy outputs $\mathbf{u}_{\textup{nom},k}\in\mathbb{R}^2$, corresponding to $[v_k,\omega_k]^\top$.
\begin{revisionone}
We set $\sum_{i\in\mathcal I_k}\beta_{i,k}=\beta=0.5$ and
$\Delta R_{\max} = 1.5(R_r+R_i^o +\Delta R_{\min})$.
\end{revisionone}

\noindent\textbf{Reward.}
Following~\cite{liu2022intention}, the reward gives $+20$ at the goal, $-20$ on collision, and otherwise a progress term $r_{\textup{prog},k}=2(d_{k-1}^{\textup{goal}}-d_k^{\textup{goal}})$ with $d_k^{\textup{goal}}=\|\mathbf{p}_k-\mathbf{p}_{\textup{goal}}\|_2$ and $\mathbf{p}_{\textup{goal}}$ denotes the goal position. 
For unicycle, we also include the rotation penalty
\(r_{\textup{rot},k}=-\alpha_s(\omega_k\Delta t)^2\) and the backward-motion
penalty \(r_{\textup{back},k}=-\alpha_b|v_k|\) if \(v_k<0\), where
\(\alpha_s,\alpha_b>0\).

\noindent\textbf{Performance metrics.}
Each test case $j$ terminates at time $T_j\le T_j^{\max}$ upon reaching the goal neighborhood ($\|\mathbf{p}_j-\mathbf{p}_{\textup{goal}}\|_2<\delta$) or upon collision. 
\begin{revisionone}
Let $m_t$, $m_s$, $m_c$, and $m_{\mathrm{to}}$ denote
the total, successful, collision, and timeout counts, respectively, with
$m_t=m_s+m_c+m_{\mathrm{to}}$, and let $\mathcal{M}$ denote the index set of successful cases.
\end{revisionone}
The episodic return is $G_j=\sum_{k=0}^{K_j}\gamma^k r_{j,k}$ where $K_j=\lfloor T_j/\Delta t \rfloor$, and the minimum distance is
$
d_j^{\min}=\min_{t\in[0,T_j]}\min_{i\in\mathcal{I}}\bigl(\|\mathbf{p}_j(t)-\mathbf{p}_i^o(t)\|-R_r-R_i^o\bigr)
$. We report five metrics:
\begin{equation*}
\begin{gathered}
\textup{SR}=\frac{m_s}{m_t},\;
\textup{CR}=\frac{m_c}{m_t},\;
\begin{revisionone}
\textup{TR}=\frac{m_{\mathrm{to}}}{m_t}\;
\end{revisionone}
\textup{Avg.\,Ret.}=\frac{1}{m_t}\!\sum_{j=1}^{m_t} G_j,
\\
\textup{Traj.\,Time}=\frac{1}{m_s}\!\sum_{j\in\mathcal{M}} T_j,\;\;
\textup{Min.\,Dist.}=\frac{1}{m_s}\!\sum_{j\in\mathcal{M}} d_j^{\min}.
\end{gathered}
\end{equation*}

% \vspace{-3pt} 
\subsection{Baselines}
Table~\ref{tab:baseline} compares 12 baselines under 20 obstacles for the unicycle robot model. The evaluated methods are:
\textbf{Opt.}: ORCA~\cite{alonso2013optimal} (rule-based reciprocal collision avoidance), CBF-QP~\cite{ames2019cbf} (deterministic safety control),~\ac{cvarbfqp}~\cite{ahmadi2021risk} (risk-aware safe control with a fixed risk level), and Adaptive-CVaR-BF~\cite{wang2025safe} (CVaR-BF with adaptive risk tuning).
\textbf{RL:}
    Vanilla~\ac{rl} is a standard end-to-end policy that outputs control actions directly from observations without explicit safety mechanisms.
    CrowdNav++ (const vel)~\cite{liu2022intention} incorporates a socially aware interaction encoder, while CrowdNav++ (inferred)~\cite{liu2022intention} extends this with a modular prediction of obstacle intentions.
\textbf{RL+Opt.}: 
We evaluate two RL+optimization variants: post-hoc safety filtering and differentiable-layer integration. In SF baselines, nominal~\ac{rl} actions are projected through~\eqref{eq:cvarbfqp} to satisfy probabilistic safety constraints. BarrierNet~\cite{xiao2023barriernet} is trained end-to-end with a differentiable CBF-QP layer, but assumes deterministic environments without stochastic obstacle uncertainty.

\begin{revisionone}
\begin{figure}[t]
\centering
\includegraphics[width=0.93\linewidth]{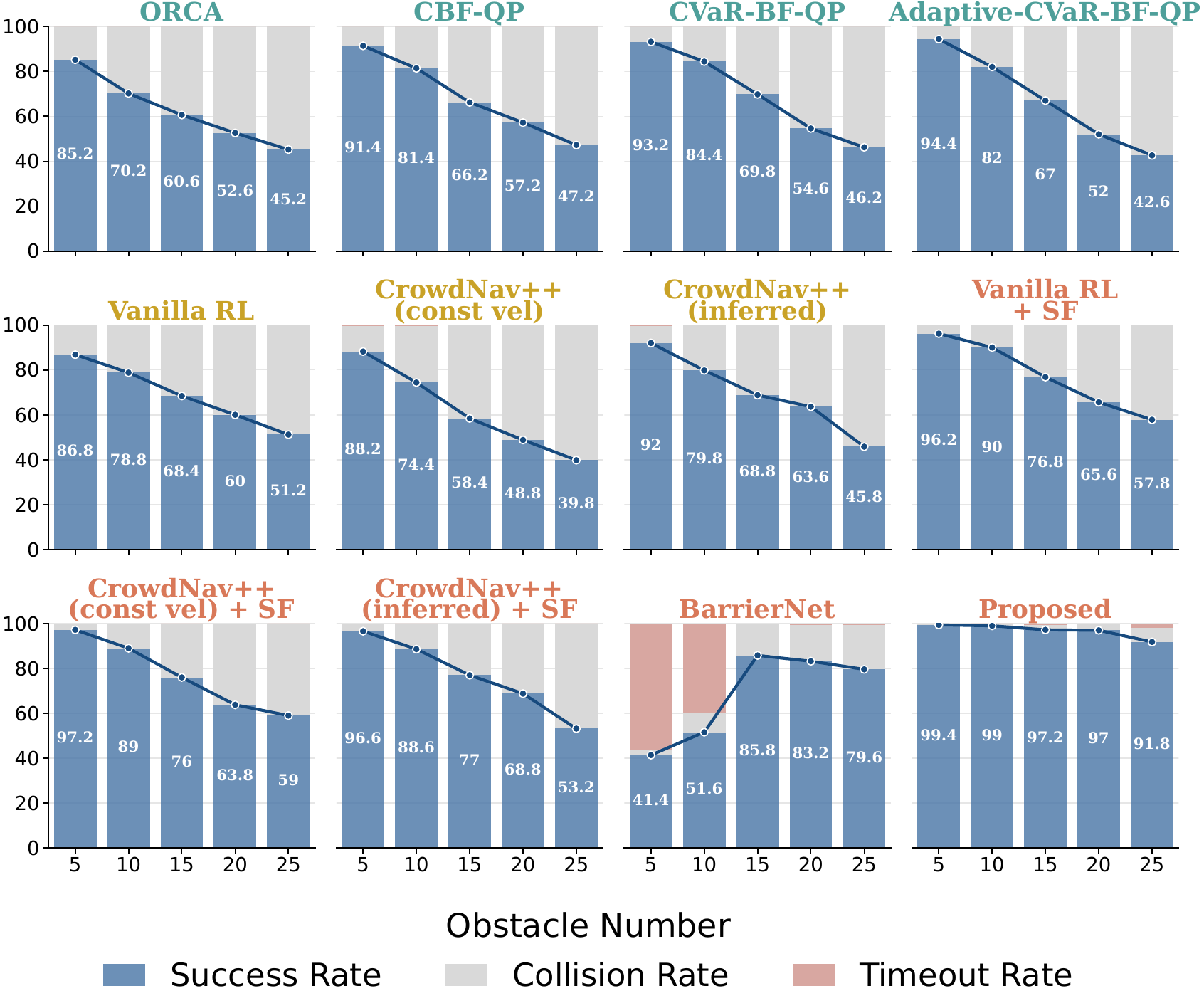}
\caption{Success rate versus obstacle number with different methods: optimization (\textcolor{optcolor}{green}),
RL (\textcolor{rlcolor}{yellow}), and integrated~\ac{rl}+ Opt. approaches
(\textcolor{hybridcolor}{red}).
}
\label{fig:sr_obstacles_unicycle}
\vspace{-15pt}
\end{figure}
\end{revisionone}

In general, our method outperforms all baselines on the key metrics, achieving the highest success rate. Although the trajectory time is not the shortest, this is expected as the reward function discourages short yet impolite behaviors. Nevertheless, the proposed method remains efficient, achieving the highest return while maintaining a comfortable distance and comparable trajectory time. Here, we further distill the following key insights from our results.

\textbf{Q1: How do different categories of methods perform in dynamic environments?}
Optimization-based methods generally struggle in highly dynamic environments because the underlying control problems can become infeasible, particularly when surrounding obstacles exhibit uncooperative behavior. In such cases, the robot cannot generate an evasive action, leading directly to collisions and lower success rates. In contrast, pure~\ac{rl} methods improve task performance in some cases but still exhibit relatively high collision rates without explicit safety constraints.
Approaches that combine~\ac{rl} with optimization leverage the strengths of both and achieve more reliable overall performance.

\textbf{Q2: How do different ways of integrating safety into~\ac{rl} affect performance?}
For~\ac{rl} methods, a post-hoc SF improves the success rate, but remains suboptimal compared to end-to-end safety integration, as it only corrects actions after learning and does not enforce safety during training, potentially degrading policy optimality. 
In contrast, both the proposed method and BarrierNet~\cite{xiao2023barriernet} jointly optimize the policy and safety constraints through a differentiable safety layer, enabling adaptive behavior. However, unlike BarrierNet, the proposed method explicitly models stochastic uncertainty and learns adaptive risk parameters, which leads to stronger performance.

% \vspace{-2pt}
\textbf{Q3: How robust are all methods across obstacle densities?}
Fig.~\ref{fig:sr_obstacles_unicycle} evaluates performance as the number of obstacles increases. As obstacle density grows, all methods exhibit a clear decline in success rate, reflecting the increased difficulty of navigation in dense environments. However, the rate of degradation varies across methods: ours is comparatively more stable, while other methods degrade more rapidly.

% \vspace{-3pt}
\subsection{Generalization}
Table~\ref{tab:ood_case_comparison} evaluates generalization under~\ac{ood} scenarios. 
We compare four representative methods spanning optimization-based,~\ac{rl}-based, post-hoc safety filtering, and differentiable-layer approaches.

\textbf{Q4: How robust are different types of methods under~\ac{ood} cases?}
Vanilla~\ac{rl} exhibits a noticeable degradation in performance, often performing worse than optimization-based methods such as~\ac{cvarbfqp}. This suggests that optimization-based approaches provide greater robustness to environmental variations compared to purely learned policies.
At the same time, incorporating optimization-based safety mechanisms into~\ac{rl} (e.g.,~\ac{rl} with SF) significantly mitigates this degradation, recovering much of the lost performance and bringing it closer to that of optimization-based methods. However, a performance gap still remains.  
The proposed method consistently outperforms the baselines using SF, achieving higher success rates and returns while maintaining competitive safety. 
This robustness arises from the end-to-end learning of adaptive risk parameters, which enables the policy to adjust its behavior in response to changing scenarios. 

\begin{revisionone}
\begin{table}[t]
\centering
\footnotesize
\setlength{\tabcolsep}{2.2pt}
\renewcommand{\arraystretch}{1.15}
\caption{Generalization performance under two~\ac{ood} scenarios. Case I: more obstacles (30 obstacles). Case II: increased obstacle radius (\(0.5\,\textup{m}\)).
}
\begin{tabular}{llcccccc}
\hline
\textbf{Case} & \textbf{Method}
& \textbf{SR} & \textbf{CR} & \textbf{TR} & \textbf{Avg. Ret.} & \textbf{Traj. Time} & \textbf{Min. Dist.} \\
\hline

\multirow{4}{*}{\textbf{I}}
&\ac{cvarbfqp}
& \noshade\worst{40.4} & \noshade\worst{59.6} & \noshade 0.0 & \noshade\worst{\meanstd{1.16}{2.31}}
& \noshade\meanstd{6.76}{1.98} & \noshade\meanstd{0.68}{0.77} \\
&\ac{rl}
& \shadefour 42.4 & \shadefour 57.6 & \shadefour 0.0 & \shadefour\meanstd{1.40}{2.24}
& \shadefour\best{\meanstd{5.85}{1.27}} & \shadefour\best{\meanstd{0.77}{0.73}} \\
&\ac{rl} + SF
& \noshade 48.2 & \noshade 51.8 & \noshade 0.0 & \noshade\meanstd{1.53}{2.33}
& \noshade\meanstd{6.33}{1.47} & \noshade\meanstd{0.75}{0.69} \\
& Proposed
& \shadefour\best{88.2}
& \shadefour\best{7.8}
& \shadefour\worst{4.0}
& \shadefour\best{\meanstd{2.53}{1.99}}
& \shadefour\worst{\meanstd{16.25}{7.36}} & \shadefour\worst{\meanstd{0.64}{0.45}} \\

\noalign{\vskip 1pt}
\cline{1-2}\cline{3-8}
\noalign{\vskip 1pt}

\multirow{4}{*}{\textbf{II}}
&\ac{cvarbfqp}
& \noshade\worst{47.2} & \noshade\worst{52.8} & \noshade 0.0 & \noshade\worst{\meanstd{1.51}{2.47}}
& \noshade\meanstd{6.87}{2.06} & \noshade\meanstd{0.79}{0.82} \\
&\ac{rl}
& \shadefour 49.6 & \shadefour 50.4 & \shadefour 0.0 & \shadefour\meanstd{1.84}{2.27}
& \shadefour\best{\meanstd{6.03}{1.32}} & \shadefour\best{\meanstd{0.86}{0.75}} \\
&\ac{rl} + SF
& \noshade 58.2 & \noshade 41.8 & \noshade 0.0 & \noshade\meanstd{2.07}{2.33}
& \noshade\meanstd{6.43}{1.48} & \noshade\meanstd{0.81}{0.72} \\
& Proposed
& \shadefour\best{95.2}
& \shadefour\best{3.2}
& \shadefour\worst{1.6}
& \shadefour\best{\meanstd{3.10}{1.40}}
& \shadefour\worst{\meanstd{14.47}{6.29}} & \shadefour\worst{\meanstd{0.73}{0.55}} \\
\hline
\end{tabular}
\label{tab:ood_case_comparison}
\vspace{-10pt}
\end{table}
\end{revisionone}

\begin{revisionone}
\begin{table}[t]
\centering
\footnotesize
\setlength{\tabcolsep}{3pt}
\renewcommand{\arraystretch}{1.15}
\caption{Ablation study on the obstacle-wise risk allocation over the active set $\{\beta_{i,k}\}_{i\in\mathcal{I}_k}$ and adaptive safety margin $\Delta R$. All GMM modes of each obstacle share its risk level \(\beta_{i,k}\). Policy 1: Proposed (adaptive risk allocation, adaptive $\Delta R$). Policy 2: fixed risk allocation ($\beta=0.5$), adaptive $\Delta R$. Policy 3: adaptive risk allocation, fixed $\Delta R=0.05$. Policy 4: fixed risk allocation ($\beta=0.5$), fixed $\Delta R=0.05$. Policy 5: w/o safety layer (Vanilla RL with graph-based observation).}
\begin{tabular}{lcccccc}
\hline
\textbf{Policy}
& \textbf{SR} & \textbf{CR} & \textbf{TR} & \textbf{Avg. Ret.} & \textbf{Traj. Time} & \textbf{Min. Dist.} \\
\hline

Policy 1
& \shadefour\best{97.0}
& \shadefour\best{2.6}
& \shadefour 0.4
& \shadefour\meanstd{3.21}{1.26}
& \shadefour\worst{\meanstd{14.42}{6.32}}
& \shadefour\meanstd{0.77}{0.57} \\
Policy 2
& \noshade 92.0
& \noshade 7.6
& \noshade 0.4
& \noshade\meanstd{3.18}{1.49}
& \noshade\meanstd{13.00}{5.60}
& \noshade\meanstd{0.77}{0.60} \\
Policy 3
& \shadefour 85.4
& \shadefour 14.0
& \shadefour 0.6
& \shadefour\worst{\meanstd{2.86}{1.83}}
& \shadefour\meanstd{13.01}{6.22}
& \shadefour\meanstd{0.73}{0.64} \\
Policy 4
& \noshade 84.8
& \noshade 14.4
& \noshade\worst{0.8}
& \noshade\meanstd{2.89}{1.88}
& \noshade\meanstd{12.82}{6.74}
& \noshade\worst{\meanstd{0.73}{0.66}} \\
Policy 5
& \shadefour\worst{80.4}
& \shadefour\worst{19.6}
& \shadefour 0.0
& \shadefour\best{\meanstd{3.31}{1.86}}
& \shadefour\best{\meanstd{6.36}{1.48}}
& \shadefour\best{\meanstd{0.80}{0.75}} \\
\hline
\end{tabular}
\label{tab:ablation}
\vspace{-10pt}
\end{table}
\end{revisionone}

% \vspace{-3pt}
\begin{revisionone}
\subsection{Ablation Study}
Table~\ref{tab:ablation} reports an ablation under the same framework that isolates the individual contributions.
Success rate degrades monotonically as adaptive components are removed, confirming that each factor contributes independently.
This reflects the complementary roles of the two mechanisms: $\beta$ reactively adjusts how much risk the policy tolerates as perceived danger changes, whereas $\Delta R$ proactively expands the physical buffer before an obstacle becomes imminent, giving the robot the space and time to act on that risk assessment. 
Without the safety layer, the robot collides far more often (CR $19.6\%$). This demonstrates the effectiveness of our differentiable~\ac{cvarbfqp} layer.
\end{revisionone}

\section{Conclusions}
This paper presented an end-to-end risk-adaptive framework for safe robot navigation in uncertain dynamic crowds. By coupling an~\ac{rl} policy with a differentiable~\ac{cvarbfqp} safety layer under a~\ac{gmm} obstacle-motion uncertainty model, the method jointly learns nominal control input, risk level, and safety margin, enabling adaptive conservatism with probabilistic safety guarantees. Experiments with the unicycle robot show that the method maintains safety while improving efficiency relative to optimization-based,~\ac{rl}-based, and integrated~\ac{rl} with optimization methods. 
The gains remain evident in the high-density~\ac{ood} case with 30 obstacles:
for \begin{revisionone}the unicycle robot\end{revisionone}, the proposed
method achieves \begin{revisionone}$2.08\times$\end{revisionone} the success
rate of Vanilla~\ac{rl}.

% \section*{Acknowledgment}
% xxx

\bibliographystyle{IEEEtran}
\bibliography{refs}

\end{document}